\documentclass{easychair}

\title{Synthesising Graphical Theories}
\titlerunning{Synthesising Graphical Theories}
\author{Aleks Kissinger}
\authorrunning{A. Kissinger}

\usepackage{tikzfig}

\tikzstyle{every picture}=[baseline=-0.25em]
\tikzstyle{dotpic}=[scale=0.6]
\tikzstyle{diredges}=[every to/.style={diredge}]
\tikzstyle{dot graph}=[shorten <=-0.1mm,shorten >=-0.1mm,scale=0.6]
\tikzstyle{digraph}=[-latex]
\tikzstyle{plot point}=[circle,fill=black,minimum width=2mm,inner sep=0]
\tikzstyle{string graph}=[scale=0.6]
\tikzstyle{sg diredge}=[-stealth]
\tikzstyle{rewrite edge}=[-open triangle 45]
\tikzstyle{sg bold diredge}=[-stealth,thick,shorten >=-1pt]
\tikzstyle{sg vertex}=[circle,minimum width=2.2mm,fill=white,draw=black,inner sep=0mm]
\tikzstyle{labelled sg vertex}=[circle,minimum width=7mm,fill=white,draw=black,inner sep=0mm]
\tikzstyle{sg grey vertex}=[sg vertex,fill=gray!30!white]
\tikzstyle{sg black vertex}=[sg vertex,fill=black]
\tikzstyle{sg bold vertex}=[circle,minimum width=2.2mm,fill=white,draw=black,very thick,inner sep=0mm]
\tikzstyle{sg wire vertex}=[circle,minimum width=1mm,fill=black,inner sep=0mm]

\tikzstyle{tick vertex}=[rectangle,fill=black,minimum height=1mm,minimum width=2.5mm,inner sep=0mm]

\tikzstyle{braceedge}=[decorate,decoration={brace,amplitude=2mm,raise=-1mm}]
\tikzstyle{small braceedge}=[decorate,decoration={brace,amplitude=1mm,raise=-1mm}]
\tikzstyle{left hook arrow}=[left hook-latex]
\tikzstyle{right hook arrow}=[right hook-latex]

\tikzstyle{dot}=[inner sep=0.7mm,minimum width=0pt,minimum height=0pt,fill=black,draw=black,shape=circle]
\tikzstyle{white dot}=[dot,fill=white]
\tikzstyle{alt white dot}=[white dot,label={[xshift=2.9mm,yshift=-0.1mm]left:$\cdot$}]
\tikzstyle{gray dot}=[dot,fill=gray!50]
\tikzstyle{box vertex}=[draw=black,rectangle]
\tikzstyle{whitebg}=[fill=white,inner sep=2pt]
\tikzstyle{graph state vertex}=[sg vertex,fill=black]

\tikzstyle{wide point}=[fill=white,draw=black,shape=isosceles triangle,shape border rotate=90,isosceles triangle stretches=true,inner sep=1pt,minimum width=1.5cm,minimum height=5mm]
\tikzstyle{wide copoint}=[fill=white,draw=black,shape=isosceles triangle,shape border rotate=-90,isosceles triangle stretches=true,inner sep=1pt,minimum width=1.5cm,minimum height=5mm]
\tikzstyle{symm}=[ultra thick,shorten <=-1mm,shorten >=-1mm]

\tikzstyle{square box}=[rectangle,fill=white,draw=black,minimum height=6mm,minimum width=6mm]
\tikzstyle{square gray box}=[rectangle,fill=gray!30,draw=black,minimum height=6mm,minimum width=6mm]
\tikzstyle{point}=[regular polygon,regular polygon sides=3,draw=black,scale=0.75,inner sep=-0.5pt,minimum width=7mm,fill=white]
\tikzstyle{copoint}=[point,regular polygon rotate=180,fill=white]
\tikzstyle{gray point}=[point,fill=gray!40!white]
\tikzstyle{gray copoint}=[copoint,fill=gray!40!white]

\tikzstyle{open graph}=[baseline=-0.25em]
\tikzstyle{greybg}=[background rectangle/.style={fill=black!5,draw=black!30,rounded corners=1ex}, show background rectangle]

\tikzstyle{edge point}=[circle,minimum width=1mm,fill=black,inner sep=0mm]
\tikzstyle{vertex point}=[circle,minimum width=2.2mm,fill=white,draw=black,inner sep=0mm]
\tikzstyle{gray vertex point}=[circle,minimum width=2.2mm,fill=gray!30!white,draw=black,inner sep=0mm]
\tikzstyle{edge label}=[inner sep=2pt, font=\small]
\tikzstyle{on edge label}=[fill=white, font=\footnotesize, inner sep=1 pt]

\newcommand{\edgearrow}{{\arrow[black]{>}}}
\newcommand{\edgetick}{{\arrow[black,scale=0.7,very thick]{|}}}
\tikzstyle{diredge}=[postaction=decorate,decoration={markings, mark=at position 0.55 with \edgearrow}]
\tikzstyle{medium diredge}=[postaction=decorate,decoration={markings, mark=at position 0.75 with \edgearrow}]
\tikzstyle{short diredge}=[->]
\tikzstyle{halfedge}=[-)]
\tikzstyle{other halfedge}=[(-]
\tikzstyle{freeedge}=[(-)]
\tikzstyle{white edge}=[line width=5pt,white]
\tikzstyle{tick}=[postaction=decorate,decoration={markings, mark=at position 0.5 with \edgetick}]
\tikzstyle{small map edge}=[|-latex, gray!60!blue, shorten <=0.9mm, shorten >=0.5mm]
\tikzstyle{thick dashed edge}=[very thick,dashed,gray!40]
\tikzstyle{map edge}=[|-latex,very thick, gray!40, shorten <=1mm, shorten >=0.5mm]
\tikzstyle{tickedge}=[postaction=decorate,
  decoration={markings, mark=at position 0.5 with \edgetick}]
\tikzstyle{dirtickedge}=[postaction=decorate,
  decoration={markings, mark=at position 0.5 with \edgetick},
  decoration={markings, mark=at position 0.85 with \edgearrow}]
\tikzstyle{dirdoubletickedge}=[postaction=decorate,
  decoration={markings, mark=at position 0.4 with \edgetick},
  decoration={markings, mark=at position 0.6 with \edgetick},
  decoration={markings, mark=at position 0.85 with \edgearrow}]

\tikzstyle{arrs}=[-latex,font=\small,auto]
\tikzstyle{arrow plain}=[arrs]
\tikzstyle{arrow dashed}=[dashed,arrs]
\tikzstyle{arrow bold}=[very thick,arrs]
\tikzstyle{arrow hide}=[draw=white!0,-]
\tikzstyle{arrow reverse}=[latex-]
\tikzstyle{cdnode}=[]

\tikzstyle{cnot}=[fill=white,shape=circle,inner sep=-1.4pt]
\tikzstyle{bang box}=[draw=black,dashed,minimum height=12mm,minimum width=12mm,fill=gray!20]
\tikzstyle{wire label}=[font=\footnotesize, auto]

\usepackage{amsmath}
\usepackage{amssymb}
\usepackage{stmaryrd}
\usepackage{xspace}
\usepackage{hyperref}
\usepackage{graphicx}
\graphicspath{{figures/}}

\usepackage{algpseudocode}
\usepackage{algorithm}

\newtheorem{theorem}{Theorem}[section]
\newtheorem*{theorem*}{Theorem}

\newtheorem{lemma}[theorem]{Lemma}

\theoremstyle{definition}
\theoremstyle{definition}
\theoremstyle{definition}\newtheorem{definition}[theorem]{Definition}
\theoremstyle{definition}\newtheorem{definitions}[theorem]{Definitions}
\theoremstyle{definition}
\theoremstyle{definition}
\theoremstyle{definition}
\theoremstyle{definition}
\theoremstyle{definition}
\theoremstyle{definition}

\newcommand{\catFRel}{\ensuremath{\textrm{\bf FRel}}\xspace}
\newcommand{\catVect}{\ensuremath{\textrm{\bf Vect}}\xspace}

\newcommand{\catGraph}{\ensuremath{\textrm{\bf Graph}}\xspace}
\newcommand{\catMat}{\ensuremath{\textrm{\bf Mat}}\xspace}

\newcommand{\catSGraph}{\ensuremath{\textrm{\bf SGraph}}\xspace}

\newcommand{\dom}{\ensuremath{\textrm{\rm dom}}}
\newcommand{\cod}{\ensuremath{\textrm{\rm cod}}}
\newcommand{\In}{\textrm{\rm In}}
\newcommand{\Out}{\textrm{\rm Out}}
\newcommand{\Bound}{\textrm{\rm Bound}}

\newcommand{\cmdrewritesto}{\tikz[baseline=-0.25em] { \draw [-open triangle 45, line width=0.2pt] (0,0) -- (0.5,0); }\,}
\newcommand{\cmdrewriteequiv}{\tikz[baseline=-0.25em] { \draw [open triangle 45-open triangle 45, line width=0.2pt] (0,0) -- node [auto,yshift=-1.2mm] {$*$} (0.7,0); }\,}
\newcommand{\cmdrewritetrans}{\tikz[baseline=-0.25em] { \draw [-open triangle 45, line width=0.2pt] (0,0) -- node [auto,pos=0.3,yshift=-1.2mm] {$*$} (0.5,0); }\,}

\DeclareMathOperator{\rewritesto}{\cmdrewritesto}
\DeclareMathOperator{\rewriteequiv}{\cmdrewriteequiv}
\DeclareMathOperator{\rewritetrans}{\cmdrewritetrans}

\tikzstyle{cdiag}=[matrix of math nodes, row sep=2em, column sep=2em, text height=1.5ex, text depth=0.25ex,inner sep=0.5em]
\tikzstyle{arrow above}=[transform canvas={yshift=0.5ex}]
\tikzstyle{arrow below}=[transform canvas={yshift=-0.5ex}]

\newcommand{\NWbracket}[1]{%
\draw #1 +(-0.25,0.5) -- +(-0.5,0.5) -- +(-0.5,0.25);}
\newcommand{\NEbracket}[1]{%
\draw #1 +(0.25,0.5) -- +(0.5,0.5) -- +(0.5,0.25);}

\newcommand{\crun}[5]{
\ensuremath{#1 \overset{#2}{\longrightarrow} #3 \overset{#4}{\longrightarrow} #5}
}

\newcommand{\cspan}[5]{
\ensuremath{#1 \overset{#2}{\longleftarrow} #3
               \overset{#4}{\longrightarrow} #5}
}

\def\bR{\begin{color}{red}} 
\def\bB{\begin{color}{blue}}
\def\bM{\begin{color}{magenta}}
\def\bC{\begin{color}{cyan}}
\def\bW{\begin{color}{white}}
\def\bBl{\begin{color}{black}} 
\def\bG{\begin{color}{green}}
\def\bY{\begin{color}{yellow}}
\def\e{\end{color}}

\begin{document}

\maketitle

\begin{abstract}
  In recent years, diagrammatic languages have been shown to be a powerful and expressive tool for reasoning about physical, logical, and semantic processes represented as morphisms in a monoidal category. In particular, categorical quantum mechanics, or ``Quantum Picturalism'', aims to turn concrete features of quantum theory into abstract structural properties, expressed in the form of diagrammatic identities. One way we search for these properties is to start with a concrete model (e.g. a set of linear maps or finite relations) and start composing generators into diagrams and looking for graphical identities.
  
  Na\"ively, we could automate this procedure by enumerating all diagrams up to a given size and check for equalities, but this is intractable in practice because it produces far too many equations. Luckily, many of these identities are not primitive, but rather derivable from simpler ones. In 2010, Johansson, Dixon, and Bundy developed a technique called \textit{conjecture synthesis} for automatically generating conjectured term equations to feed into an inductive theorem prover. In this extended abstract, we adapt this technique to diagrammatic theories, expressed as graph rewrite systems, and demonstrate its application by synthesising a graphical theory for studying entangled quantum states.
\end{abstract}

\vspace{-3mm}\section{Introduction}\vspace{-2mm}
\label{sec:intro}

Monoidal categories can be thought of as theories of generalised (physical, logical, semantic, ...) processes. In particular, they provide an abstract setting for studying how processes (i.e. morphisms) behave when they are composed in time (via the categorical composition $\circ$) and in space (via the monoidal product $\otimes$). Categorical quantum mechanics, a program initiated by Abramsky and Coecke in 2004, takes this notion seriously by showing that many phenomena in quantum mechanics can be completely understood at the abstract level of monoidal categories carrying certain extra structure~\cite{AC2004}. This notion has also had success in reasoning about stochastic networks~\cite{CoeckeSpekkens2011}, illustrative ``toy theories'' of physics~\cite{EdwardsSpek2011}, and even linguistics~\cite{CoeckeLinguistics2010}. In most of these contexts, we make extensive use of spacial ``swap'' operations (symmetries) and temporal ``feedback'' operations (traces). Formally, this means we are working with a particular kind of monoidal category called a symmetric traced category.

Working with large compositions of morphisms in a symmetric traced category using traditional, term-based expressions quickly becomes unwieldy, as one attempts to describe an inherently two-dimensional structure using a (one-dimensional) term language. \textit{Diagrammatic} languages thus provide a natural alternative.

String diagrams represent morphisms as ``boxes'' with various inputs and outputs, and compositions as ``wires'' connecting those boxes. They provide a way of visualising composition that is both highly intuitive and completely rigorous. They were originated by Penrose in the 1970s as a way to describe contractions of (abstract) tensors~\cite{Penrose1971}. These abstract tensor systems closely resembled traced monoidal categories, so perhaps unsurprisingly, Joyal and Street were able to use string diagrams in 1991 to construct free symmetric traced categories~\cite{JS}. In the context of categorical quantum mechanics, complementary observables~\cite{Coecke2008}, quantum measurements~\cite{DuncanPerdrix2010}, and entanglement~\cite{KissingerICALP2010} can be studied using string diagrams and diagrammatic identities.

Joyal and Street's formalisation of string diagrams uses topological graphs with some extra structure. While this method comes very close to the intuitive notion of a string diagram, it is not obvious how one could automate the creation and manipulation of these topological objects in a computer program. For that reason, Dixon, Duncan, and Kissinger developed a discrete representation of string diagrams called \textit{string graphs}. Unlike their continuous counterparts, it is straightforward to automate the manipulation of string graphs using graph rewrite systems. Using a graph rewrite system, we can draw conclusions about any concrete model of that system. Since the amount of data needed to represent a morphism concretely often grows exponentially in the number of inputs and outputs, we can gain potentially huge computational benefits by manipulating these morphisms using rewrite rules instead of concrete calculation.

We might also ask if this can be done in reverse. Consider a situation where we have a collection of generators as well as concrete realisations of those generators (e.g. as linear maps, finite relations, etc.) and we wish to discover the algebraic structure of these generators. One way to find new rules is to start enumerating all string graphs involving those generators, evaluating them concretely, and checking for identities. This process becomes intractable very quickly, as the number of distinct string graphs over a set of generators grows exponentially with size. However, in 2010, Johansson, Dixon, and Bundy provided a clever way of avoiding enumerating redundant rewrite rule candidates for a term rewrite theory, which they called \textit{conjecture synthesis}~\cite{Johansson:2010tk}. This technique keeps a running collection of rewrite rules and only enumerates terms that are \textit{irreducible} with respect to this collection of rules. This simple condition is surprisingly effective at curbing the exponential blow-up of the rewrite system synthesis procedure, and provides the basis for a piece of software developed by the authors called IsaCoSy~\cite{IsaCoSy} (for ISAbelle COnjecture SYnthesis).

The main contribution of this paper is an algorithm that generates new graphical theories using an adapted conjecture synthesis procedure. This adapted procedure has been implemented in a tool called QuantoCoSy, built on string graph rewriting platform called Quantomatic~\cite{Quantomatic}. We provide an example of applying QuantoCoSy to generate a graphical theory used in the study of quantum entanglement. We use as a basis the string graph (aka open-graph) formalism developed in~\cite{DixonKissinger2010}.

After briefly reviewing string graphs and string graph rewriting in section~\ref{sec:string-graphs}, we show in section~\ref{sec:interpreting} how, given a valuation of the generators of a string graph, the graph itself can be evaluated as a linear map using tensor contraction. The graphical conjecture synthesis procedure is then described in section~\ref{sec:quanto-cosy}. In section~\ref{sec:application}, we look at a test case where our theory synthesis software, QuantoCosy, is used to synthesise a simple graphical theory of entangled states.

\vspace{-3mm}\section{Background: String Graphs}\vspace{-2mm}
\label{sec:string-graphs}

In this section, we summarise the construction of the category $\catSGraph_T$ of string graphs parametrised by a monoidal signature. Details of how this can be done formally in the context of adhesive categories can be found in~\cite{DixonKissinger2010}. Our goal in defining string graphs is to represent diagrammatic theories using graph rewrite systems. To do this, we need to re-cast (topological) string diagrams as typed graphs. We do this by replacing \textit{wires} with chains of special vertices called \textit{wire-vertices}. We also introduce \textit{node-vertices}, which represent ``logical'' nodes (aka boxes) in the diagram.
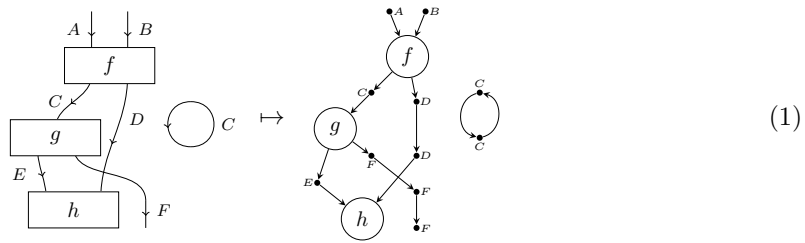
\begin{equation}\label{eq:string-graph}
  \scalebox{0.8}{%
\beginpgfgraphicnamed{string_diagram_to_graph}
\begin{tikzpicture}[string graph]
	\begin{pgfonlayer}{nodelayer}
		\node [style=none] (0) at (-4.5, 3) {};
		\node [style=none] (1) at (-3.5, 3) {};
		\node [style=sg wire vertex] (2) at (3.75, 3) {};
		\node [style=none, font=\tiny] (3) at (4, 3) {$A$};
		\node [style=sg wire vertex] (4) at (4.75, 3) {};
		\node [style=none, font=\tiny] (5) at (5, 3) {$B$};
		\node [style=none] (6) at (-4.5, 1.75) {};
		\node [style=none] (7) at (-3.5, 1.75) {};
		\node [style=square box, minimum width=1.5 cm] (8) at (-4, 1.5) {$f$};
		\node [style=labelled sg vertex] (9) at (4.25, 1.75) {$f$};
		\node [style=none] (10) at (-4.5, 1.25) {};
		\node [style=none] (11) at (-3.5, 1.25) {};
		\node [style=none, font=\tiny] (12) at (3, 0.75) {$C$};
		\node [style=sg wire vertex] (13) at (3.25, 0.75) {};
		\node [style=none, font=\tiny] (14) at (6.25, 1) {$C$};
		\node [style=none] (15) at (-1.75, 0.5) {};
		\node [style=sg wire vertex] (16) at (6.25, 0.75) {};
		\node [style=sg wire vertex] (17) at (4.5, 0.5) {};
		\node [style=none, font=\tiny] (18) at (4.75, 0.5) {$D$};
		\node [style=none] (19) at (-5.5, -0.25) {};
		\node [style=square box, minimum width=1.5 cm] (20) at (-5.5, -0.5) {$g$};
		\node [style=none] (21) at (0.5, 0) {\large $\mapsto$};
		\node [style=labelled sg vertex] (22) at (2.25, -0.25) {$g$};
		\node [style=none] (23) at (-6, -0.75) {};
		\node [style=none] (24) at (-5, -0.75) {};
		\node [style=none] (25) at (-1.75, -0.75) {};
		\node [style=sg wire vertex] (26) at (4.5, -1) {};
		\node [style=none, font=\tiny] (27) at (4.75, -1) {$D$};
		\node [style=sg wire vertex] (28) at (6.25, -0.5) {};
		\node [style=sg wire vertex] (29) at (3.25, -1) {};
		\node [style=none, font=\tiny] (30) at (6.25, -0.75) {$C$};
		\node [style=none, font=\tiny] (31) at (1.5, -1.75) {$E$};
		\node [style=sg wire vertex] (32) at (1.75, -1.75) {};
		\node [style=none, font=\tiny] (33) at (3.25, -1.25) {$F$};
		\node [style=none] (34) at (-5.75, -2.25) {};
		\node [style=none] (35) at (-4.25, -2.25) {};
		\node [style=sg wire vertex] (36) at (4.5, -2) {};
		\node [style=none, font=\tiny] (37) at (4.75, -2) {$F$};
		\node [style=square box, minimum width=1.5 cm] (38) at (-5, -2.5) {$h$};
		\node [style=labelled sg vertex] (39) at (3, -2.75) {$h$};
		\node [style=none] (40) at (-3, -2.25) {};
		\node [style=none] (41) at (-3, -3) {};
		\node [style=sg wire vertex] (42) at (4.5, -3) {};
		\node [style=none, font=\tiny] (43) at (4.75, -3) {$F$};
		\node [style=none, wire label] (44) at (-5, 2.5) {$A$};
		\node [style=none, wire label] (45) at (-3, 2.5) {$B$};
		\node [style=none, wire label] (46) at (-5.5, 0.5) {$C$};
		\node [style=none, wire label] (47) at (-3.25, 0) {$D$};
		\node [style=none, wire label] (48) at (-2.5, -2.5) {$F$};
		\node [style=none, wire label] (49) at (-6.5, -1.5) {$E$};
	\end{pgfonlayer}
	\begin{pgfonlayer}{edgelayer}
		\draw [style=sg diredge, in=-15, out=15, looseness=1.25] (28) to (16);
		\draw [style=sg diredge] (22) to (32);
		\draw [style=diredge, in=90, out=-90] (11.center) to (35.center);
		\draw [style=sg diredge] (36) to (42);
		\draw [style=sg diredge] (13) to (22);
		\draw [style=sg diredge] (9) to (17);
		\draw [style=diredge, in=180, out=180, looseness=1.75] (15.center) to (25.center);
		\draw [style=sg diredge] (26) to (39);
		\draw [in=90, out=-90] (24.center) to (40.center);
		\draw [style=diredge, in=90, out=-90] (10.center) to (19.center);
		\draw [in=0, out=0, looseness=1.75] (15.center) to node[wire label]{$C$} (25.center);
		\draw [style=sg diredge] (29) to (36);
		\draw [style=sg diredge, in=165, out=-165, looseness=1.25] (16) to (28);
		\draw [style=sg diredge] (22) to (29);
		\draw [style=sg diredge] (2) to (9);
		\draw [style=sg diredge] (4) to (9);
		\draw [style=diredge] (0.center) to (6.center);
		\draw [style=sg diredge] (32) to (39);
		\draw [style=sg diredge] (9) to (13);
		\draw [style=diredge] (40.center) to (41.center);
		\draw [style=diredge, in=90, out=-90] (23.center) to (34.center);
		\draw [style=diredge] (1.center) to (7.center);
		\draw [style=sg diredge] (17) to (26);
	\end{pgfonlayer}
\end{tikzpicture}}
\endpgfgraphicnamed}
\end{equation}

A (small, strict) monoidal signature $T$ defines a set of generators that can be interpreted in a monoidal category. Let $w(X)$ be the set of lists over a set $X$. $T$ consists of a set $\mathcal M$ of \textit{morphisms}, a set $\mathcal O$ of objects and functions $\dom, \cod : \mathcal M \rightarrow w(\mathcal O)$ that assign input and output types for each morphism. Note in particular that $T$ says nothing about composition.

\begin{definition}
  The category $\catSGraph_T$ has as objects string graphs with node-vertex and wire-vertex types given by $T$, and as arrows graph homomorphisms respecting these types.
\end{definition}

From hence forth, we will refer to chains of wire vertices simply as wires, when there is no ambiguity. There are a few things to note here. Firstly, wire-vertices carry the (object) type of a wire, so every connection between two node-vertices must contain at least one node-vertex. We do not allow wires to split or merge, so wire-vertices must have at most one input and one output. When our generators are non-commutative, we distinguish inputs and outputs using edge types ($\textrm{in}_1$, $\textrm{in}_2$, $\textrm{out}_1$, $\textrm{out}_2$, $\ldots$).

This category can be defined as a full subcategory of the slice category $\catGraph / G_T$, where $G_T$ is called the \textit{derived typegraph} of a monoidal signature $T$. Furthermore, $\catSGraph_T$ is a \textit{partial adhesive category}, as defined in~\cite{DixonKissinger2010}. Such a category provides the mechanisms we need to perform graph matching and rewriting in the following sections. For details, see~\cite{DixonKissinger2010}.

\begin{definitions}
  For a string graph $G$, let $\omega(G)$ be the set of wire-vertices and $\eta(G)$ the set of node-vertices. If a wire-vertex has no in-edges, it is called an \emph{input}. We write the set of inputs of a string graph $G$ as $\In(G)$. Similarly, a wire-vertex with no out-edges is called an \emph{output}, and the set of outputs is written $\Out(G)$. The inputs and outputs define a string graph's \emph{boundary}, written $\Bound(G)$. If a boundary vertex has no in-edges and no out-edges, (it is both and input and output) it is called an \emph{isolated wire-vertex}.
\end{definitions}





\vspace{-2mm}\subsection{Plugging and Rewriting for String Graphs}\vspace{-1mm}
\label{sec:rewriting-string-graphs}

There are two important, related operations we wish to apply to string graphs. The first is plugging, where input and output wire-vertices are glued together, and the second is string graph rewriting, where parts of the string graph are cut out and replaced with other graphs.

\begin{definition}\rm
  For a string graph $G$ and wire-vertices $x \in \In(G)$ and $y \in \Out(G)$, let $G/_{\!(x,y)}$ be a new string graph with the wire-vertices $x$ and $y$ identified. This is called the \textit{plugging} of $(x,y)$ in $G$, and the new vertex $(x\sim y)\in G/_{(x,y)}$ is called the \textit{plugged wire-vertex}.
\end{definition}

We can compose string graphs by performing pluggings on the disjoint union: $(G + H)/_{\!x,y}$ for $x \in \Out(G)$, $y \in \In(H)$. We can plug many wires together by performing many pluggings one after another. Equivalently, we can perform a pushout along a discrete graph, respecting certain conditions. When there can be no confusion, we also call this a plugging.

\begin{definition}\rm
  A pushout of a span \cspan{G}{p}{K}{q}{H} is called a \textit{plugging} if $K$ is a disconnected graph of wire-vertices, $p$ and $q$ are mono, and for all $k\in K$, $p(k) \in \Bound(G)$, $q(k) \in \Bound(H)$ and $p(k) \in \In(G) \Leftrightarrow q(k) \in \Out(H)$.
\end{definition}

We perform rewrites on string graphs using the \textit{double pushout} (DPO) graph rewriting technique. We first define a string graph rewrite rule as a pair of string graphs sharing a boundary. It will be convenient to formalise this as a particular kind of span.

\begin{definition}\label{def:rw-rule} \rm
  A \textit{string graph rewrite rule} $L \rightarrow R$ is a span of monomorphisms \cspan{L}{b_1}{B}{b_2}{R} such that $b_1(B) = \Bound(L)$, $b_2(B) = \Bound(R)$ and for all $b\in B$, $b_1(b) \in \In(L) \Leftrightarrow b_2(b) \in \In(R)$.
\end{definition}


To perform rewriting, we first introduce a notion of matchings. These are special monomorphisms that only connect to the rest of the graph via their boundary.

\begin{definition}\rm
  A \textit{matching} of a string graph $L$ on another string graph $G$ is a monomorphism $m : L \rightarrow G$ such that whenever an edge not in the image of $m$ is adjacent to $m(v)$ for some vertex $v$ in $L$, $v$ must be in the boundary $\Bound(L)$ of $L$. When such a matching exists, we say $L$ (or any rule with $L$ as its LHS) \textit{matches} $G$.
\end{definition}

Note that, for the class of rewrite rules considered (i.e. those of the form of Definition~\ref{def:rw-rule}), this corresponds to the notion of matching defined by Ehrig et al~\cite{Ehrig2008}. Once a matching is found, the occurrence of $L$ in $G$ can be replaced with $R$. However, some care must be taken to ensure that $R$ is reconnected to the remainder of $G$ in the same location as $L$ was previously. We begin by removing the \textit{interior} of $L$ in $G$, i.e. the part of $L$ that is not in the boundary. We do this by finding a \textit{context} graph $G'$ such that when $L$ and $G'$ are plugged together along the boundary of $L$, the total graph $G$ is obtained. This is known as computing the \textit{pushout complement} of \crun{B}{b_1}{L}{m}{G}. Since $G$ is the result of gluing $L$ and $G'$ together along $B$, we can replace $L$ with $R$ by performing a second plugging, this time of $R$ and $G'$ along $B$.
\begin{equation}\label{eq:sg-dpo}
  \begin{tikzpicture}
    \matrix (m) [cdiag] {
      L & B  & R \\
      G & G' & H \\
    };
    \path [arrs]
      (m-1-2) edge [left hook-latex] (m-1-1)
      (m-1-1) edge node [swap] {$m$} (m-2-1)
      (m-1-2) edge [right hook-latex] (m-1-3)
      (m-2-2) edge (m-2-1)
      (m-2-2) edge (m-2-3)
      (m-1-2) edge node {$m'$} (m-2-2)
      (m-1-3) edge (m-2-3);
    \NEbracket{(m-2-1)}
    \NWbracket{(m-2-3)}
  \end{tikzpicture}
\end{equation}

In other words, we express $G$ as a plugging of two graphs $L$ and $G'$, then compute $H$ by plugging $R$ into $G'$. This completes the rewrite of $G$ into $H$. Note that even in a category with all pushouts, pushout complements need not exist or be unique. However, the following theorem, proved in \cite{DixonKissinger2010}, shows that DPO rewriting is always well-defined for string graph matchings and rewrite rules.

\begin{theorem}
  Let \cspan{L}{b_1}{B}{b_2}{R} be a boundary span and $m : L \rightarrow G$ a matching. Then \crun{B}{b_1}{L}{m}{G} has a unique pushout complement $G'$ and both of the pushout squares in diagram (\ref{eq:sg-dpo}) exist and are preserved by the embedding of $\catSGraph_T$ into $\catGraph/G_T$.
\end{theorem}

We will also find the following lemma useful when combining matching and string graph enumeration. It shows that performing a plugging only creates new matchings \textit{local} to the plugged vertex.

\begin{lemma}\label{lem:match-local}
  Let $x \in \In(G)$, $y \in \Out(G)$, then let $m : L \rightarrow G/_{\!(x,y)}$ be a matching such that the plugged vertex $x\sim y$ is not in the image of $m$. Then there exists a matching $m' : L \rightarrow G$.
\end{lemma}

\begin{proof}
  Let $q : G \rightarrow G/_{\!(x,y)}$ be the quotient map. Then, pulling back $q$ over $m$ is just the restriction of $q$ to the image of $m$. Since the image of $m$ does not contain the plugged vertex, $q$ restricts to an isomorphism $r$. It is straightforward to show that $m' \circ r^{-1}$ is a matching.
\end{proof}


\begin{definition}
  A set of string graph rewrite rules is called a \textit{string graph rewrite system}. For a string graph rewrite system $\mathbb S$, we write $G \rewritesto_{\mathbb S} H$ if there exists a rule $L \rewritesto R \in \mathbb S$ and a matching $m : L \rightarrow G$ such that $G$ can be rewritten to $H$ using the DPO diagram given by diagram (\ref{eq:sg-dpo}). Let $\rewritetrans_{\mathbb S}$ be the reflexive, transitive closure of $\rewritesto_{\mathbb S}$ and $\rewriteequiv_{\mathbb S}$ be the associated equivalence relation.
\end{definition}


The vigilant reader will notice that we allow any number of wire-vertices to occur within a wire (all with the same type). However, the number of wire-vertices makes no semantic difference in the interpretation of a string graph. Therefore, we will consider two string graphs to be semantically equivalent if they only differ in the length of their wires. This equivalence relation is called \textit{wire-homeomorphism}. We can formalise wire-homeomorphism as a string graph rewrite system.

\begin{definition}\label{def:homeo-rewrite-system}\rm
  For a monoidal signature $T = (\mathcal O, \mathcal M, \dom, \cod)$, the rewrite system $\mathbb H$ is defined as follows. For every $X \in \mathcal O$, we define a loop contraction rule $h^L_X$ and a wire contraction rule $h^W_X$. For every $f \in \mathcal M$ and $0 \leq i < \textrm{Length}(\dom(f))$, $0 \leq j < \textrm{Length}(\cod(f))$, we define an input contraction rule $h^I_{f,i}$ and an output contraction rule $h^O_{f,j}$.
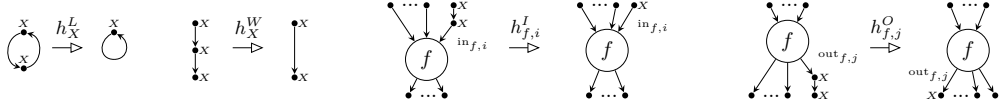
  \begin{center}
    \scalebox{0.8}{%
\beginpgfgraphicnamed{wire_homeo1}
\begin{tikzpicture}[string graph]
	\begin{pgfonlayer}{nodelayer}
		\node [style=sg wire vertex] (0) at (1, 0.75) {};
		\node [style=none, font=\tiny] (1) at (1.25, 0.75) {$X$};
		\node [style=none, font=\tiny] (2) at (-3.75, 0.75) {$X$};
		\node [style=sg wire vertex] (3) at (-3.75, 0.5) {};
		\node [style=none, font=\small, yshift=-0.7 mm] (4) at (-2.5, 0.75) {$h^L_X$};
		\node [style=none, font=\tiny] (5) at (-1.25, 0.75) {$X$};
		\node [style=none, font=\small, yshift=-0.7 mm] (6) at (2.5, 0.75) {$h^W_X$};
		\node [style=sg wire vertex] (7) at (3.75, 0.75) {};
		\node [style=none, font=\tiny] (8) at (4, 0.75) {$X$};
		\node [style=sg wire vertex] (9) at (-1.25, 0.5) {};
		\node [style=none] (10) at (-2.5, 0) {$\rewritesto$};
		\node [style=sg wire vertex] (11) at (1, 0) {};
		\node [style=none, font=\tiny] (12) at (1.25, 0) {$X$};
		\node [style=none] (13) at (2.5, 0) {$\rewritesto$};
		\node [style=sg wire vertex] (14) at (-3.75, -0.5) {};
		\node [style=none] (15) at (-1.25, -0.25) {};
		\node [style=sg wire vertex] (16) at (3.75, -0.75) {};
		\node [style=none, font=\tiny] (17) at (4, -0.75) {$X$};
		\node [style=none, font=\tiny] (18) at (-3.75, -0.25) {$X$};
		\node [style=sg wire vertex] (19) at (1, -0.75) {};
		\node [style=none, font=\tiny] (20) at (1.25, -0.75) {$X$};
	\end{pgfonlayer}
	\begin{pgfonlayer}{edgelayer}
		\draw [style=sg diredge] (7) to (16);
		\draw [in=180, out=-150, looseness=1.50] (9) to (15.center);
		\draw [style=sg diredge] (0) to (11);
		\draw [style=sg diredge, in=-30, out=0, looseness=1.50] (15.center) to (9);
		\draw [style=sg diredge] (11) to (19);
		\draw [style=sg diredge, bend right=75, looseness=1.25] (14) to (3);
		\draw [style=sg diredge, bend right=75, looseness=1.25] (3) to (14);
	\end{pgfonlayer}
\end{tikzpicture}}
\endpgfgraphicnamed\qquad\quad%
\beginpgfgraphicnamed{wire_homeo2}
\begin{tikzpicture}[string graph]
	\begin{pgfonlayer}{nodelayer}
		\node [style=sg wire vertex] (0) at (-5.25, 1.25) {};
		\node [style=none, font=\tiny] (1) at (-5, 1.25) {$X$};
		\node [style=sg wire vertex] (2) at (-1.75, 1.25) {};
		\node [style=none] (3) at (-1.25, 1.25) {\,...};
		\node [style=sg wire vertex] (4) at (-0.75, 1.25) {};
		\node [style=sg wire vertex] (5) at (-0.25, 1.25) {};
		\node [style=none, font=\tiny] (6) at (0, 1.25) {$X$};
		\node [style=sg wire vertex] (7) at (-7, 1.25) {};
		\node [style=none] (8) at (-6.5, 1.25) {\,...};
		\node [style=sg wire vertex] (9) at (-6, 1.25) {};
		\node [style=sg wire vertex] (10) at (-5.25, 0.75) {};
		\node [style=none, font=\tiny] (11) at (-5, 0.75) {$X$};
		\node [style=none, font=\small, yshift=-0.7 mm] (12) at (-3.25, 0.75) {$h^I_{f,i}$};
		\node [style=none, font=\tiny, anchor=west, yshift=1 mm, xshift=-1 mm] (13) at (0, 0.5) {$\textrm{in}_{f,i}$};
		\node [style=none, font=\small, yshift=-0.7 mm] (14) at (6.75, 0.75) {$h^O_{f,j}$};
		\node [style=none, font=\tiny, anchor=west, yshift=1 mm, xshift=-1 mm] (15) at (-5, 0) {$\textrm{in}_{f,i}$};
		\node [style=none] (16) at (-3.25, 0) {$\rewritesto$};
		\node [style=labelled sg vertex] (17) at (-1, 0) {$f$};
		\node [style=none, font=\tiny, anchor=west, yshift=-1 mm, xshift=-1 mm] (18) at (5, 0) {$\textrm{out}_{f,j}$};
		\node [style=none] (19) at (6.75, 0) {$\rewritesto$};
		\node [style=labelled sg vertex] (20) at (-6, -0.25) {$f$};
		\node [style=none, font=\tiny, anchor=west, yshift=-1 mm, xshift=-1 mm] (21) at (7.5, -0.5) {$\textrm{out}_{f,j}$};
		\node [style=sg wire vertex] (22) at (-1.5, -1.25) {};
		\node [style=none] (23) at (-1, -1.25) {\,...};
		\node [style=sg wire vertex] (24) at (-0.5, -1.25) {};
		\node [style=sg wire vertex] (25) at (-6.5, -1.25) {};
		\node [style=none] (26) at (-6, -1.25) {\,...};
		\node [style=sg wire vertex] (27) at (-5.5, -1.25) {};
		\node [style=sg wire vertex] (28) at (9.25, -1.25) {};
		\node [style=sg wire vertex] (29) at (4.5, 1.25) {};
		\node [style=none] (30) at (4, 1.25) {\,...};
		\node [style=sg wire vertex] (31) at (3.5, 1.25) {};
		\node [style=sg wire vertex] (32) at (8.5, 1.25) {};
		\node [style=none] (33) at (3.5, -1.25) {\,...};
		\node [style=sg wire vertex] (34) at (9.5, 1.25) {};
		\node [style=labelled sg vertex] (35) at (4, 0.25) {$f$};
		\node [style=none] (36) at (8.75, -1.25) {\,...};
		\node [style=sg wire vertex] (37) at (4.75, -0.75) {};
		\node [style=none] (38) at (9, 1.25) {\,...};
		\node [style=sg wire vertex] (39) at (9.75, -1.25) {};
		\node [style=sg wire vertex] (40) at (3, -1.25) {};
		\node [style=sg wire vertex] (41) at (8.25, -1.25) {};
		\node [style=sg wire vertex] (42) at (4.75, -1.25) {};
		\node [style=labelled sg vertex] (43) at (9, 0) {$f$};
		\node [style=sg wire vertex] (44) at (4, -1.25) {};
		\node [style=none, font=\tiny] (45) at (5, -0.75) {$X$};
		\node [style=none, font=\tiny] (46) at (5, -1.25) {$X$};
		\node [style=none, font=\tiny] (47) at (8, -1.25) {$X$};
	\end{pgfonlayer}
	\begin{pgfonlayer}{edgelayer}
		\draw [style=sg diredge] (0) to (10);
		\draw [style=sg diredge] (9) to (20);
		\draw [style=sg diredge] (17) to (24);
		\draw [style=sg diredge] (10) to (20);
		\draw [style=sg diredge] (20) to (27);
		\draw [style=sg diredge] (2) to (17);
		\draw [style=sg diredge] (4) to (17);
		\draw [style=sg diredge] (5) to (17);
		\draw [style=sg diredge] (20) to (25);
		\draw [style=sg diredge] (7) to (20);
		\draw [style=sg diredge] (17) to (22);
		\draw [style=sg diredge] (37) to (42);
		\draw [style=sg diredge] (35) to (44);
		\draw [style=sg diredge] (34) to (43);
		\draw [style=sg diredge] (35) to (37);
		\draw [style=sg diredge] (29) to (35);
		\draw [style=sg diredge] (43) to (41);
		\draw [style=sg diredge] (43) to (28);
		\draw [style=sg diredge] (43) to (39);
		\draw [style=sg diredge] (31) to (35);
		\draw [style=sg diredge] (35) to (40);
		\draw [style=sg diredge] (32) to (43);
	\end{pgfonlayer}
\end{tikzpicture}}
\endpgfgraphicnamed}
  \end{center}
\end{definition}

It was shown in~\cite{DixonKissinger2010} that this is a confluent, terminating rewrite system. Formal forms are called \textit{reduced string graphs}, and contain only one wire-vertex on every wire.


\vspace{-3mm}\section{Interpreting String Graphs as Tensor Contractions}\vspace{-2mm}
\label{sec:interpreting}


Describing the contractions of many tensors was the main reason Penrose introduced string diagram notation, so perhaps unsurprisingly, there is a natural way to interpret a string graph as a contraction of tensors. We first give the monoidal signature $T = (\mathcal O, \mathcal M, \dom, \cod)$ a \textit{valuation} $v : T \rightarrow \catVect_{\mathbb C}$. This is a monoidal signature homomorphism from $T$ to $\catVect_{\mathbb C}$ that assigns to each of the objects $o \in \mathcal O$ a vector space and each of the morphisms $f \in \mathcal M$ a linear map $v(f)$. We can then regard the linear map $v(f)$ as a tensor with a lower index for every input to $f$ in $T$ and an upper index for every output.

A string graph can then be interpreted as a big tensor contraction by interpreting wire-vertices as identities (i.e. the Dirac delta tensor $\delta_i^j$), node-vertices as linear maps $v(f)$, for $f$ the type of the vertex, and edges as sums over an index.
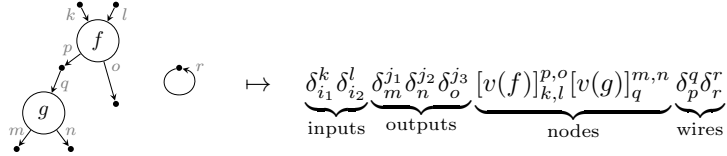
\[ \scalebox{0.8}{%
\beginpgfgraphicnamed{string_graph_ex}
\begin{tikzpicture}[string graph]
	\begin{pgfonlayer}{nodelayer}
		\node [style=sg wire vertex] (0) at (-1.25, -1.75) {};
		\node [style=sg wire vertex] (1) at (-2.75, -1.75) {};
		\node [style=sg wire vertex] (2) at (-1, 2.25) {};
		\node [style=labelled sg vertex] (3) at (-0.5, 1.25) {$f$};
		\node [style=sg wire vertex] (4) at (1.75, 0.5) {};
		\node [style=labelled sg vertex] (5) at (-2, -0.75) {$g$};
		\node [style=none] (6) at (1.75, -0.25) {};
		\node [style=sg wire vertex] (7) at (0, -0.5) {};
		\node [style=sg wire vertex] (8) at (0, 2.25) {};
		\node [style=sg wire vertex] (9) at (-1.5, 0.5) {};
		\node [style=none] (10) at (-1.25, 2) {\footnotesize\color{gray} $k$};
		\node [style=none] (11) at (0.25, 2) {\footnotesize\color{gray} $l$};
		\node [style=none] (12) at (-1.25, -1.25) {\footnotesize\color{gray} $n$};
		\node [style=none] (13) at (-2.75, -1.25) {\footnotesize\color{gray} $m$};
		\node [style=none] (14) at (0, 0.5) {\footnotesize\color{gray} $o$};
		\node [style=none, xshift=-0.5 mm] (15) at (-1.25, 1) {\footnotesize\color{gray} $p$};
		\node [style=none, xshift=0.5 mm] (16) at (-1.5, 0) {\footnotesize\color{gray} $q$};
		\node [style=none, xshift=0.5 mm] (17) at (2.25, 0.5) {\footnotesize\color{gray} $r$};
	\end{pgfonlayer}
	\begin{pgfonlayer}{edgelayer}
		\draw [style=sg diredge] (2) to (3);
		\draw [style=sg diredge] (3) to (7);
		\draw [style=sg diredge] (3) to (9);
		\draw [style=sg diredge] (5) to (1);
		\draw [style=sg diredge] (5) to (0);
		\draw [in=180, out=-165, looseness=1.75] (4) to (6.center);
		\draw [style=sg diredge, in=-15, out=0, looseness=1.75] (6.center) to (4);
		\draw [style=sg diredge] (8) to (3);
		\draw [style=sg diredge] (9) to (5);
	\end{pgfonlayer}
\end{tikzpicture}}
\endpgfgraphicnamed}\quad \mapsto \quad
   \underbrace{\delta_{i_1}^k \delta_{i_2}^l}_{\textrm{inputs}}
   \underbrace{\delta_m^{j_1} \delta_n^{j_2} \delta_o^{j_3}}_{\textrm{outputs}}
   \underbrace{[v(f)]_{k,l}^{p,o}
   [v(g)]_q^{m,n}}_{\textrm{nodes}}
   \underbrace{\delta_p^q \delta_r^r}_{\textrm{wires}}
\]

We are using the Einstein summation convention (repeated indices are summed over), and we have labeled edges in the string graph with their corresponding indices. Like wire-vertices themselves, the $\delta$ maps are used mainly just for book-keeping. They maintain the correct order of inputs and outputs, define circles, and connect more interesting tensors together.

This gives the evaluation of a string diagram $G$ as a linear map $\llbracket G\rrbracket$ in $\catVect_{\mathbb C}$. This is a special case of a much more general construction. In~\cite{KissingerThesis}, Kissinger showed that string graphs could be used to form the free symmetric traced category over a monoidal signature. As a consequence, for \textit{any} symmetric traced category $\mathcal C$, a monoidal signature homomorphism $v : T \rightarrow \mathcal C$ lifts uniquely to an evaluation functor from the category of string graphs to $\mathcal C$.

\vspace{-3mm}\section{Conjecture Synthesis for String Graph Theories}\vspace{-2mm}
\label{sec:quanto-cosy}

We will now describe a synthesis procedure for graphical identities in the spirit of IsaCoSy~\cite{IsaCoSy}. A notable difference is that, rather than passing conjectured identities off to an inductive theorem prover, we simply evaluate them using a given valuation of the generators of the theory.

The key to this procedure is that it maintains a set of reduction rules $\mathbb S$ throughout, and avoids enumerating string graphs that are reducible with respect to $\mathbb S$. We refer to such graphs as \textit{redexes}.

The theory synthesis procedure takes as input a string-graph signature $T$ along with an $(m,n)$-tensor for every box in $T$ with $m$ inputs and $n$ outputs. It also takes an initial set of rewrite rules $\mathbb S$ and a reduction ordering $\kappa$. This is a function from string graphs to $\mathbb N$ where $G \cong H \Rightarrow \kappa(G) = \kappa(H)$ and $G \rewritesto_{\mathbb S} H \Rightarrow \kappa(G) > \kappa(H)$. We shall also chose $\kappa$ to be non-increasing in the number of node-vertices and wire-vertices in a string graph. We will also maintain a set $\mathbb K$ (initially empty) of \textit{congruences}, but these cannot be used for redex-elimination.

In the term case, a single round of synthesis is parametrised by two natural numbers: the maximum term size (or term depth) and the maximum number of free variables occurring in the term. For string graphs, we parametrise a run with four natural numbers: the number of inputs $M$, the number of outputs $N$, the maximum number of pluggings $P$, and the maximum number of node-vertices $Q$. Will shall refer to string graphs with $M$ inputs, $N$ outputs, and up to $P,Q$ pluggings and node-vertices as \textit{string graphs of size $(M,N,P,Q)$}.


First, we define a string graph enumeration procedure that avoids reducible expressions with respect to our initial rewrite system $\mathbb S$. Suppose we have a string graph signature $T = (\mathcal O, \mathcal M, \dom, \cod)$, then we can define a string graph for every morphism in $T$, called its generator. For each $f \in \mathcal M$, this is the smallest string graph containing a node-vertex of type $f$. For completeness, we also include the identity generator, which is just two connected wire-vertices. For an example of a set of generators, see (\ref{eq:ghz-w-gens}) in the next section.

We enumerate string graphs by starting with disconnected string graphs, i.e. disjoint unions of generators, and performing pluggings to connect separate components together.

\begin{definition}\rm
  For a string graph $G$, let $\mathfrak p(G)$ be the set $\In(G) \times \Out(G)$. A pair $(x,y) \in \mathfrak p(G)$ defines a particular plugging $G/_{\!(x,y)}$. Two pluggings in $(x,y),(x',y') \in \mathfrak p(G)$ are called \textit{similar} if there exists an isomorphism $\phi : G/_{\!(x,y)} \cong G/_{\!(x',y')}$ that is identity on node-vertices. Let $\mathfrak p_{(x,y)}(G) \subseteq \mathfrak p(G)$ be the set of all pairs $(x',y')$ similar to $(x,y)$.
\end{definition}

Let $\mathcal D(M,N,Q)$ be the set of all disconnected string graphs with $M$ inputs, $N$ outputs and up to $Q$ node-vertices. The procedure {\sc ENUM} takes as input a graph, a set of (distinguishable) pluggings and the number of pluggings left to do. Note that each plugging decreases $|\In(G)|$ and $|\Out(G)|$ by $1$. So, if we start with $G \in \mathcal D(M+p,N+p,Q)$ and before $p$ pluggings, we will get a string graph with $M$ inputs and $N$ outputs.

\begin{center}
  \begin{algorithmic}[1]
    \Procedure{ENUM}{$\Pi$, $G$, $p$}
      \If{$p = 0$}
        \State{save $G$}
        \State{\textbf{exit}}
      \ElsIf{$\Pi = \{\}$}
        \State{\textbf{exit}}
      \EndIf
      \State{$\textbf{let }(x,y) \in \Pi$}
      \If{no rule in $\mathbb S$ matches $G/_{\!(x,y)}$, local to $(x,y)$}
      \label{ln:redex-guard}
        \State{{\sc ENUM}($\Pi - \{(x,y)\}$, $G/_{\!(x,y)}$, $p-1$)}
        \label{ln:do-plug}
      \EndIf
      \State{$\textbf{let }\Pi' = \Pi - \mathfrak p_{(x,y)}(G)$}
      \label{ln:similar}
      \State{{\sc ENUM}($\Pi'$, $G$, $p$)}
      \label{ln:dont-plug}
    \EndProcedure
    \State{}
    \ForAll{$0 \leq p \leq P$, $G \in \mathcal D(M+p,N+p,Q)$}
      \State{{\sc ENUM}($\mathfrak p(G)$, $G$, $p$)}
    \EndFor
  \end{algorithmic}
\end{center}

The procedure {\sc ENUM} recurses down two branches, one where we decide to do a particular plugging $(x,y) \in \Pi$ (line \ref{ln:do-plug}), and one where we decide not to (line \ref{ln:dont-plug}). Line~\ref{ln:similar} prevents us from deciding not to do a certain plugging then later deciding to do one that is similar, as this will enumerate redundant graphs.

The crucial redex-elimination step in the algorithm occurs at line \ref{ln:redex-guard}. Recall that a rule $L \rewritesto R$ matches $G/_{\!(x,y)}$ \textit{local to $(x,y)$} if the matching $m : L \rightarrow G/_{\!(x,y)}$ contains the plugged vertex $(x\sim y) \in G/_{\!(x,y)}$ in its image. Assuming each of our generators is irreducible, it suffices to consider only matches of rules local to the most recent plugging in order to eliminate all reducible expressions. By Lemma~\ref{lem:match-local}, if a rule has a matching on $G/_{\!(x,y)}$ that is not local to $(x,y)$, it already has a matching on $G$. Therefore, the condition on line \ref{ln:redex-guard} guarantees no redexes, with respect to $\mathbb S$, will be enumerated.

Once all of the irreducible graphs of size $(M,N,P,Q)$ have been enumerated, we update the rewrite system as follows. First, we evaluate the string graphs as linear maps and organise them into equivalence classes, up to scalar factors and permutations of inputs and outputs. We then filter out any remaining isomorphic graphs in each equivalence class and identify the set of minimal string graphs $C' \subseteq C$ with respect to $\kappa$. Choose a representative $G_0$ of $C'$. Finally, we add new reductions $G \rewritesto G_0$ to $\mathbb S$ for all $G \in C - C'$. Add congruences $G' \rewritesto G_0$ and $G_0 \rewritesto G'$ for all $G' \in C' - \{ G_0 \}$ to $\mathbb K$.

We postpone filtering out isomorphic string graphs until after enumeration because tensor contraction is fast, and two string graphs will not be isomorphic unless they are in the same equivalence class. In order to get a well-behaved rewrite system, we should choose $\kappa$ so that there are very few congruences, if any. Obviously, if we choose $\kappa$ to be strict (for all $G \not\cong H$, $\kappa(G) < \kappa(H)$ or $\kappa(G) > \kappa(H)$), there will be no congruences. While strict reduction orderings for graphs are much more difficult to compute than their term analogues, this may be tractable if we adapted our graph enumeration procedure to only produce canonical representatives of isomorphism classes of string graphs (cf.~\cite{Goldberg1992,McKay1998}).

Once a single run of the synthesis procedure is complete, we can re-run the procedure using larger values of $M$, $N$, $P$, and $Q$ as well as the updated rewrite system $\mathbb S$. Using this growing collection of reductions can be very effective in stemming the exponential blow-up in both the number of string graphs that need to be enumerated and the number of rewrite rules found.

\begin{theorem}
  Applying the synthesis procedure for a series of runs given by $\{ (M_i,N_i,P_i,Q_i) \}$ yields a rewrite system $\mathbb S \cup \mathbb K$ that is complete for all graphs of size $(M_i,N_i,P_i,Q_i)$ for some $i$. Furthermore, if $\mathbb K = \{\}$ the rewrite system yields unique normal forms for graphs of the given sizes.
\end{theorem}

\begin{proof}
  For completeness, suppose there are string graphs $G$ and $H$ of size $(M_i,N_i,P_i,Q_i)$ such that $\llbracket G\rrbracket = \llbracket H\rrbracket$. Since $\mathbb S$ is a terminating rewrite system, we can normalise $G$ and $H$ with respect to $\mathbb S$. Since $\kappa$ is non-increasing on node-vertices and wire-vertices, the associated normal forms $G'$ and $H'$ are also of size $(M_i,N_i,P_i,Q_i)$, so both graphs will have been enumerated. Therefore, either $G' \cong H'$ or $(G' \rewritesto H') \in \mathbb K$. For unique normal forms, note that when $\mathbb K$ is empty, only $G' \cong H'$ is possible.
\end{proof}

Of course, we are interested in graphs of all sizes. However, completeness is not expected (and sometimes not desirable) in these cases, as the model or models used to synthesise the theory may be degenerate in some sense as examples of some algebraic structure. To ensure unique normal forms, we would need to ensure confluence of $\mathbb S$. Unlike termination, this does not come for free, as there may exist critical pairs that are larger than any of the graphs synthesised. However, the synthesis procedure could be used in conjection with a graphical variant of Knuth-Bendix completion to increase the chances of obtaining a confluent rewrite system for graphs of all sizes.




\vspace{-3mm}\section{Application and Results}\vspace{-2mm}\label{sec:application}


The procedure described in section~\ref{sec:quanto-cosy} has been implemented in a tool called QuantoCosy. This sits on top a general framework for string graph rewriting called Quantomatic~\cite{Quantomatic}. We show in this section how QuantoCosy performs in synthesising an example theory, called the GHZ/W-calculus. For a detailed description of this theory and how it can be applied, see~\cite{CoeckeKissinger2010}.

In the GHZ/W-calculus, we can express large, many-body entangled quantum states as compositions of simpler components. The key here is to use three-body states as algebraic ``building blocks'' for more complex states. A three-system, or tripartite entangled state, can be represented as a vector in $H \otimes H \otimes H$ for some complex vector space $H$. For finite dimensions, we can equivalently express this vector as a linear map $H \otimes H \rightarrow H$. When $H \cong \mathbb C^2$, there are only two ``interesting'' tripartite entangled states, the GHZ state and the W state. Regarding these as linear maps from $\mathbb C^2 \otimes \mathbb C^2 \cong \mathbb C^4$ to $\mathbb C^2$, they can be written as follows:
\[
\beginpgfgraphicnamed{ghz_gen}
\begin{tikzpicture}[string graph,scale=0.8]
	\begin{pgfonlayer}{nodelayer}
		\node [style=sg vertex] (0) at (0, 0) {};
		\node [style=sg wire vertex] (1) at (-0.5, 0.75) {};
		\node [style=sg wire vertex] (2) at (0.5, 0.75) {};
		\node [style=sg wire vertex] (3) at (0, -0.75) {};
	\end{pgfonlayer}
	\begin{pgfonlayer}{edgelayer}
		\draw [style=sg diredge] (1) to (0);
		\draw [style=sg diredge] (2) to (0);
		\draw [style=sg diredge] (0) to (3);
	\end{pgfonlayer}
\end{tikzpicture}}
\endpgfgraphicnamed :=
    \left(\begin{matrix}
      1 & 0 & 0 & 0 \\
      0 & 0 & 0 & 1
    \end{matrix}\right)
    \qquad\qquad
\beginpgfgraphicnamed{w_gen}
\begin{tikzpicture}[string graph,scale=0.8]
	\begin{pgfonlayer}{nodelayer}
		\node [style=sg grey vertex] (0) at (0, 0) {};
		\node [style=sg wire vertex] (1) at (-0.5, 0.75) {};
		\node [style=sg wire vertex] (2) at (0.5, 0.75) {};
		\node [style=sg wire vertex] (3) at (0, -0.75) {};
	\end{pgfonlayer}
	\begin{pgfonlayer}{edgelayer}
		\draw [style=sg diredge] (1) to (0);
		\draw [style=sg diredge] (2) to (0);
		\draw [style=sg diredge] (0) to (3);
	\end{pgfonlayer}
\end{tikzpicture}}
\endpgfgraphicnamed :=
    \left(\begin{matrix}
      0 & 1 & 1 & 0 \\
      0 & 0 & 0 & 1
    \end{matrix}\right)
\]

These maps exhibit particularly nice algebraic identities. They can both be extended to \textit{commutative Frobenius algebras}, which are unital, associative algebras over a vector space (or an object in some monoidal category) that have a strong \textit{self-duality} property. In particular, they can be naturally associated with a co-unital, co-algebra that interacts well with the algebra. For pairs of Frobenius algebras, we can also always construct a special map called the \textit{dualiser}. See~\cite{CoeckeKissinger2010} for details on how these things are defined algebraically.

The generators of the GHZ/W-calculus are: (i) the generators of both Frobenius algebras, (ii) the dualiser of the two algebras, and (iii) two zero vectors.
\begin{equation}\label{eq:ghz-w-gens}
\beginpgfgraphicnamed{ghz_w_gens}
\begin{tikzpicture}[string graph]
	\begin{pgfonlayer}{nodelayer}
		\node [style=sg vertex] (0) at (-7.5, 0) {};
		\node [style=sg wire vertex] (1) at (-7.75, 0.5) {};
		\node [style=sg wire vertex] (2) at (-7.25, 0.5) {};
		\node [style=sg wire vertex] (3) at (-7.5, -0.5) {};
		\node [style=none] (4) at (-8.5, -0.75) {};
		\node [style=none] (5) at (-8.5, 0.75) {};
		\node [style=none] (6) at (7, -0.75) {};
		\node [style=none] (7) at (7, 0.75) {};
		\node [style=none] (8) at (-6.75, -0.5) {$,$};
		\node [style=sg vertex] (9) at (-6.25, 0.25) {};
		\node [style=sg wire vertex] (10) at (-6.25, -0.25) {};
		\node [style=none] (11) at (-5.75, -0.5) {$,$};
		\node [style=sg wire vertex] (12) at (-5.25, -0.5) {};
		\node [style=sg wire vertex] (13) at (-4.75, -0.5) {};
		\node [style=sg wire vertex] (14) at (-5, 0.5) {};
		\node [style=sg vertex] (15) at (-3.75, -0.25) {};
		\node [style=sg vertex] (16) at (-5, 0) {};
		\node [style=sg wire vertex] (17) at (-3.75, 0.25) {};
		\node [style=none] (18) at (-4.25, -0.5) {$,$};
		\node [style=none] (19) at (-3.25, -0.5) {$,$};
		\node [style=sg wire vertex] (20) at (-2.5, -0.5) {};
		\node [style=sg wire vertex] (21) at (-1.25, -0.25) {};
		\node [style=sg grey vertex] (22) at (-1.25, 0.25) {};
		\node [style=none] (23) at (-1.75, -0.5) {$,$};
		\node [style=sg wire vertex] (24) at (0, 0.5) {};
		\node [style=none] (25) at (-0.75, -0.5) {$,$};
		\node [style=sg grey vertex] (26) at (1.25, -0.25) {};
		\node [style=sg wire vertex] (27) at (-0.25, -0.5) {};
		\node [style=sg wire vertex] (28) at (-2.75, 0.5) {};
		\node [style=sg grey vertex] (29) at (-2.5, 0) {};
		\node [style=sg wire vertex] (30) at (-2.25, 0.5) {};
		\node [style=none] (31) at (1.75, -0.5) {$,$};
		\node [style=sg wire vertex] (32) at (1.25, 0.25) {};
		\node [style=sg grey vertex] (33) at (0, 0) {};
		\node [style=none] (34) at (0.75, -0.5) {$,$};
		\node [style=sg wire vertex] (35) at (0.25, -0.5) {};
		\node [style=tick vertex] (36) at (2.25, 0) {};
		\node [style=sg wire vertex] (37) at (2.25, -0.5) {};
		\node [style=sg wire vertex] (38) at (2.25, 0.5) {};
		\node [style=none] (39) at (2.75, -0.5) {$,$};
		\node [style=point] (40) at (3.5, 0.25) {$0$};
		\node [style=copoint] (41) at (5, -0.25) {$0$};
		\node [style=none] (42) at (4.25, -0.5) {$,$};
		\node [style=sg wire vertex] (43) at (3.5, -0.5) {};
		\node [style=sg wire vertex] (44) at (5, 0.5) {};
		\node [style=none] (45) at (5.75, -0.5) {$,$};
		\node [style=sg wire vertex] (46) at (6.25, 0.5) {};
		\node [style=sg wire vertex] (47) at (6.25, -0.5) {};
	\end{pgfonlayer}
	\begin{pgfonlayer}{edgelayer}
		\draw [style=sg diredge] (1) to (0);
		\draw [style=sg diredge] (2) to (0);
		\draw [style=sg diredge] (0) to (3);
		\draw [style=small braceedge] (4.center) to (5.center);
		\draw [style=small braceedge] (7.center) to (6.center);
		\draw [style=sg diredge] (9) to (10);
		\draw [style=sg diredge] (16) to (12);
		\draw [style=sg diredge] (16) to (13);
		\draw [style=sg diredge] (14) to (16);
		\draw [style=sg diredge] (17) to (15);
		\draw [style=sg diredge] (28) to (29);
		\draw [style=sg diredge] (30) to (29);
		\draw [style=sg diredge] (29) to (20);
		\draw [style=sg diredge] (22) to (21);
		\draw [style=sg diredge] (33) to (27);
		\draw [style=sg diredge] (33) to (35);
		\draw [style=sg diredge] (24) to (33);
		\draw [style=sg diredge] (32) to (26);
		\draw [style=sg diredge] (38) to (36);
		\draw [style=sg diredge] (36) to (37);
		\draw [style=sg diredge] (40) to (43);
		\draw [style=sg diredge] (44) to (41);
		\draw [style=sg diredge] (46) to (47);
	\end{pgfonlayer}
\end{tikzpicture}}
\endpgfgraphicnamed
\end{equation}

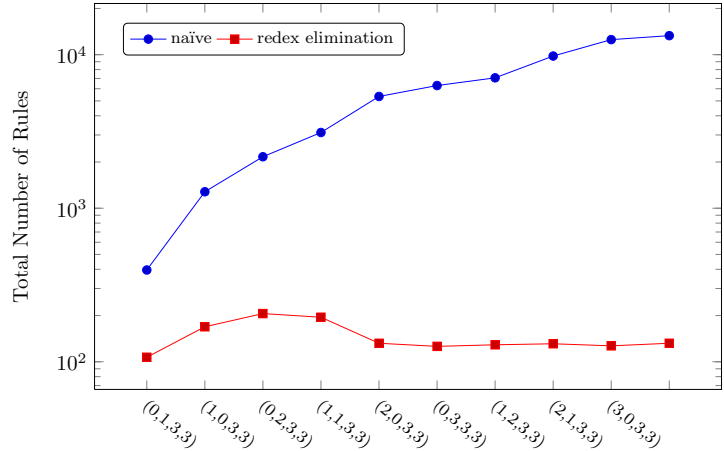
\begin{figure}
  \centering
  \scalebox{0.8}{%
\beginpgfgraphicnamed{naive_vs_redex_plot}
\begin{tikzpicture}
    \begin{semilogyaxis}[
        ylabel=Total Number of Rules,
        x label style={yshift=-8mm},
        x tick label style={anchor=west,rotate=-40,yshift=-2mm,font=\footnotesize},
        xticklabels={(0,0,3,3),(0,1,3,3),(1,0,3,3),(0,2,3,3),(1,1,3,3),
                     (2,0,3,3),(0,3,3,3),(1,2,3,3),(2,1,3,3),(3,0,3,3)},
        width=12cm,
        height=8cm,
        legend style={
            at={(0.27,0.95)},
            anchor=north,
            legend columns=2,
            cells={anchor=west},
            font=\footnotesize,
            rounded corners=2pt
        }]
        
        \addplot plot coordinates {
            (0,396)
            (1,1280)
            (2,2164)
            (3,3113)
            (4,5345)
            (5,6294)
            (6,7066)
            (7,9798)
            (8,12530)
            (9,13302)
        };
        \addplot plot coordinates {
            (0,107)
            (1,169)
            (2,206)
            (3,195)
            (4,132)
            (5,126)
            (6,129)
            (7,131)
            (8,127)
            (9,132)
        };
        
        \legend{na\"ive,redex elimination}
    \end{semilogyaxis}
\end{tikzpicture}}
\endpgfgraphicnamed}
  \caption{\label{fig:plot} Number of rules syntesised across 10 runs of the synthesis procedure.}
\end{figure}

As was the case for terms, filtering out redexes has a huge impact on the number of string graphs that need to be checked. In ten synthesis runs, we generate all of the GHZ/W rewrite rules with total node-vertices, pluggings, and inputs+outputs $\leq 3$. Using a na\"ive synthesis, this yielded $13,302$ rewrite rules. Using the procedure described in the previous section, this yielded $132$ rules, including all of the rules used to define the GHZ/W-calculus in~\cite{CoeckeKissinger2010}. A plot of the number of rewrite rules generated using a na\"ive graph enumeration algorithm against the number generated using the redex-elimination procedure across all $10$ runs is provided in Figure~\ref{fig:plot}.

\vspace{-3mm}\section{Conclusion and Outlook}\vspace{-2mm}\label{sec:conclusion}

In this paper, we showed how string diagrams could be discretised into string graphs, which can be manipulated automatically using DPO graph rewriting. We then demonstrated how a technique developed for term theories called conjecture synthesis can be adapted to work for string graphs. Finally, we showed an example of the application of this technique to a real graphical theory, called the GHZ/W-calculus. The results there were promising, as we demonstrated an exponential drop-off in the number of extraneous rules generated when using redex-eliminating routine as compared to a na\"ive synthesis routine.

In the end, some of the $132$ rules produced were seemingly arbitrary, whereas others reflect a real algebraic relationship between the GHZ-structure and the W-structure. There are various reasons, including simple aesthetics, why a human mathematician would take some of these rewrites to be valuable or interesting. QuantoCosy is completely ignorant to such considerations, as it only knows what rules to throw away, rather than which ones to highlight. However, this is already useful, as it provides a researcher with a pool of hundreds of rules to sort through, rather than tens of thousands.

There are various ways in which this can be improved further. One is to make the synthesis smarter by employing heuristics that search for particular classes of rewrites that are common to many algebraic structures. This has been done for terms using \textit{scheme-based} conjecture synthesis~\cite{MontanoRivas2010}, where the search procedure starts by attempting to prove familiar sets of identities (associativity and commutativity of a binary operation, for instance) before moving to less familiar conjectures.

Another way to improve the quality and conciseness of the synthesis output would be to refine the synthesised rules using Knuth-Bendix completion. This could prove a powerful method for automatically producing new rewrite rules for \textit{pattern graphs}~\cite{KissingerThesis}. Pattern graphs are used to describe infinite families of string graphs that have some repeated substructure. Starting with a rewrite system that contains some pattern graph rewrite rules that are previously known, and some concrete rules that are produced using conjecture synthesis, one can obtain new pattern graph rewrites by performing critical pair analysis and completion.

Finally, and most importantly, this procedure can be improved and its scope broadened by looking at different kinds of concrete models. As mentioned briefly in section~\ref{sec:interpreting}, string graphs can be used to construct the free symmetric traced category over a signature. So, any interpretation of a signature in a concrete symmetric traced category can be lifted uniquely to an evaluation of string graphs. Examples of such categories where one might wish to form models are:
  $\catMat(F)$ of matrices over a finite field,
  $(\catFRel, \times)$ of sets and finite relations with the ``wave'' style tensor and trace,
  $(\catFRel, +)$ of sets and finite relations with the ``token'' style tensor and trace,
  and products of concrete symmetric traced categories (i.e. giving valuations in many models simultaneously).


The support of a theory synthesis tool allows a mathematician to very quickly get a picture of a theory by specifying a few generators, and opens to the way for experimentation and rapid development of a wide variety of new theories.

\bibliographystyle{akbib}
\bibliography{bibfile}

\end{document}